\documentclass[11pt]{article}
\usepackage{amsmath,amsthm,amssymb}

\usepackage[a4paper,left=3.5cm,right=3.3cm,top=3.2cm,bottom=3cm]{geometry}

\usepackage{algorithm,graphicx}
\usepackage{mathrsfs}
\usepackage{xcolor}

\numberwithin{equation}{section}
\newtheorem{thm}{Theorem}[section]

\newtheorem{cor}[thm]{Corollary}
\theoremstyle{definition}

\newtheorem{exm}{Example}[section]
\theoremstyle{remark}
\newtheorem{rmk}{Remark}[section]

\usepackage[fleqn,tbtags]{mathtools}
\usepackage{natbib}

\newcommand\blfootnote[1]{%
  \begingroup
  \renewcommand\thefootnote{}\footnote{#1}%
  \addtocounter{footnote}{-1}%
  \endgroup
}

\allowdisplaybreaks

\title{Regret Bounds for Lifelong Learning}
\author{Pierre Alquier$^1$ \and The Tien Mai$^1$ \and Massimiliano Pontil$^2$
\blfootnote{The authors gratefully acknowledge financial support from Labex ECODEC (ANR - 11-LABEX-0047).
Pierre Alquier gratefully acnkowledges financial support from the research
programme {\it New Challenges for New Data} from LCL and GENES, hosted by
the {\it Fondation du Risque}}}
\date{%
    $^1$CREST, ENSAE, Universit\'e Paris Saclay\\%
    $^2$Istituto Italiano di Tecnologia  and  University College London\\[2ex]%
    \today
}

\begin{document}

\maketitle

\begin{abstract}
We consider the problem of transfer learning in an online setting. Different tasks are presented sequentially 
and processed by a within-task algorithm. We propose a lifelong learning strategy which refines the underlying 
data representation used by the within-task algorithm, thereby transferring information from one task to the next. 
We show that when the within-task algorithm comes with some regret bound, our strategy inherits this good property. 
Our bounds are in expectation for a general loss function, and uniform for a convex loss. We discuss applications to dictionary learning and finite set of predictors. In the latter case, we improve previous $O(1/\sqrt{m})$ bounds 
to $O(1/m)$ where $m$ is the per task sample size.
\end{abstract}

\section{INTRODUCTION}
%Transferring knowledge gained from previously learned tasks is crucial for  learning a new similar task, especially when the sample size is small. This is  the essence of transfer learning approach, which can massively improve the performance over learning in isolation. 

Most analyses of learning algorithms assume that the algorithm starts learning from scratch 
when presented with a new dataset. However, in real life, it is often the case that we will use the 
same algorithm on many different tasks, and that information should be transferred from one task to
another. For example, a key problem in pattern recognition is to learn a dictionary of features helpful
for image classification: it makes perfectly sense to assume that features learnt to classify dogs against other 
animals can be re-used to recognize cats. This idea is at the core of {\it transfer learning}, see~\citep{thrun,Balcan,Baxter97,Baxter2000,Cavallanti,Maurer05,maurer2013sparse,pentina2014pac,MPR2016} and references therein. 

The setting in which the tasks are presented simultaneously is often referred to as {\em learning-to-learn} \citep{Baxter2000}, whereas when the tasks are presented sequentially, the term {\it lifelong learning} is often used~\citep{thrun96}. In either case, a huge improvement over ``learning in isolation'' can be expected, 
especially when the sample size per task is relatively small. We will use the above terminologies in the paper.
%even though we note it is not always used consistently in the literature.

Although a substantial amount of work has been done on the theoretical study of learning-to-learn \citep{Baxter2000,Maurer05,pentina2014pac,MPR2016}, up to our knowledge there 
is no analysis of the statistical performance of lifelong learning algorithms. \citet{ruvolo2013ella} 
studied the convergence of certain optimization algorithms for lifelong learning. However, no statistical 
guarantees are provided. Furthermore, in all the aforementioned works, the authors propose a technique for transfer learning which constrains the within-task algorithm to be of a certain kind, e.g.~regularized empirical risk minimization.
%least squares. 

The main goal of this paper is to show that it is possible to perform a theoretical analysis of lifelong learning 
with minimal assumptions on the form of the within-task algorithm. Given a learner with her/his own favourite 
algorithm(s) for learning within tasks, we propose a meta-algorithm for transferring information
from one task to the next. The algorithm maintains a prior distribution on the set of representations, which is updated after the encounter of each new task using the exponentially weighted aggregation (EWA) procedure, hence we call it {\it EWA for lifelong learning} or EWA-LL.

A standard way to provide theoretical guarantees for online algorithms are regret bounds, which measure the discrepancy between the prediction error of the forecaster and the error of an ideal predictor. We prove that, as long as the within-task algorithms have good statistical properties, EWA-LL inherits these properties. Specifically in Theorem \ref{thm:online:w:online} we present regret bounds for EWA-LL, in which the regret bounds for the within-tasks algorithms are combined into a regret bound for the meta-algorithm.

We also show, using an online-to-batch analysis, that it is possible to derive a strategy for learning-to-learn, 
and provide risk bounds for this strategy. The bounds are generally in the order of $1/\sqrt{T}+1/\sqrt{m}$, 
where $T$ is the number of tasks and $m$ is the sample size per task. Moreover, we derive in some specific
situations rates in $1/\sqrt{T}+1/m$. These rates are novel up to our knowledge and justify the use of transfer 
learning with very small sample sizes $m$.

The paper is organized as follows. In Section~\ref{notations} we introduce the lifelong learning problem. 
In Section~\ref{EWA-LL} we present the EWA-LL algorithm and provide a bound on its expected regret. 
This bound is very general, but might be uneasy to understand at first sight. So, in Section~\ref{applications} 
we present more explicit versions of our bound in two classical examples: finite set of predictors and dictionary 
learning. We also provide a short simulation study for dictionary learning. At this point, we hope that the 
reader will have a clear overview of the problem under study. The rest of the paper is devoted to theoretical 
refinements: in online learning, uniform bounds are the norm rather than bounds in expectations~\citep{cesa2006prediction}. In Section~\ref{uniform} we establish such bounds for EWA-LL. 
Section~\ref{online-to-batch} provides an online-to-batch analysis that allows one to use a modification of 
EWA-LL for learning-to-learn. The supplementary material include proofs (Appendix~\ref{proofs}),  
improvements for dictionary learning (Appendix~\ref{appendix:improvement}) and 
extended results (Appendix~\ref{appendix-bwo}).
%discussion of a slightly different setting for lifelong learning (Appendix~\ref{appendix-bwo}).

%%%%%%%%%%%%%%%%%%%%%%%%%%%%%%%%

\section{PROBLEM}
\label{notations}
In this section, we introduce our notation and present the lifelong learning problem.

Let $\mathcal{X}$ and $\mathcal{Y}$ be some sets. A predictor is a function $f:\mathcal{X}\rightarrow \mathcal{Y}$, where $\mathcal{Y} = \mathbb{R}$ for regression and $\mathcal{Y}=\{-1,1\}$ for binary classification. 
The loss of a predictor $f$ on a pair $(x,y)$ is a real number denoted by $\ell(f(x),y)$.
As mentioned above, we want to transfer the information (a common data representation) gained from the 
previous tasks to a new one. Formally, we let $\mathcal{Z}$ be a set and prescribe a set $\mathcal{G}$ of 
feature maps (also called {\em representations}) $g:\mathcal{X}\rightarrow \mathcal{Z}$, and a set $\mathcal{H}$ of functions $h:\mathcal{Z}\rightarrow \mathbb{R}$. We shall design an algorithm that is useful when
there is a function $g\in\mathcal{G}$, common to all the tasks, and task-specific functions $h_1,\dots,h_T$ such that 
\[
f_t= h_t \circ g
\]
is a good predictor for task $t$, in the sense that the corresponding prediction error (see below) is small.

% and our aim is to minimize the regret.

We are now ready to describe the learning problem. We assume that tasks are dealt with sequentially. Furthermore, we assume that each task dataset is itself revealed sequentially and refer to this setting as {\em online-within-online} lifelong learning. Specifically, at each time step $t\in\{1,\dots,T\}$, 
the learner is challenged with a task, corresponding to a dataset
$$ \mathcal{S}_t = 
\big((x_{t,1},y_{t,1}),\dots,(x_{t,m_t},y_{t,m_t})\big) \in 
(\mathcal{X} \times \mathcal{Y})^{m_t}$$
where $m_t \in {\mathbb N}$. The dataset $\mathcal{S}_t$ is itself revealed sequentially, that is, at each inner step $i \in \{1,\dots,m_t\}$:
\begin{itemize}
 \item The object $x_{t,i}$ is revealed,
 \item The learner has to predict $y_{t,i}$, let
 $\hat{y}_{t,i} $ denote the prediction,
 \item The label $y_{t,i}$ is revealed, and the learner incurs the loss $\hat{\ell}_{t,i} := \ell(\hat{y}_{t,i},y_{t,i})$.
\end{itemize}
%MMM in thhe also above g and h are not use. This is fine but may raise some confusion
The task $t$ ends at time $m_t$, at which point the prediction error is 
\begin{equation}
\frac{1}{m_t} \sum_{i=1}^{m_t}\hat{\ell}_{t,i}.
\label{eq:mmm}
\end{equation}
This process is repeated for each task $t$, so that at the end of all the tasks, the average error is
$$
\frac{1}{T} \sum_{t=1}^T \frac{1}{m_t} \sum_{i=1}^{m_t}\hat{\ell}_{t,i}.
$$
Ideally, if for a given representation $g$, the best predictor $h_t$ for task $t$ was known in advance, then an ideal learner using $h_t \circ g$ for prediction would incur the error
\begin{equation}
\inf_{h_t \in\mathcal{H}} \frac{1}{m_t} \sum_{i=1}^{m_t}
   \ell\big(h_t\circ g (x_{t,i}),y_{t,i}\big).
\label{eq:mmm2}
\end{equation}
Hence, we define the within-task-regret of the representation $g$ on task $t$ as the difference between the prediction error \eqref{eq:mmm} and the smallest prediction error \eqref{eq:mmm2}, $$
\mathcal{R}_t(g)
=
\frac{1}{m_t}\sum_{i=1}^{m_t}\hat{\ell}_{t,i}
- \inf_{h_t \in\mathcal{H}}\frac{1}{m_t} \sum_{i=1}^{m_t}
   \ell\big(h_t\circ g (x_{t,i}),y_{t,i}
   \big).
$$
The above expression is slightly different from the usual notion of regret \cite{cesa2006prediction}, which does not contain the factor $1/m_t$.
This normalization is important in that it allows us to give equal weigths to different tasks.
 
Note that an oracle who would have known the best common representation $g$ for all tasks in advance would have only suffered, on the entire sequence of datasets, the error
$$
\inf_{g\in\mathcal{G}}
\frac{1}{T}
\sum_{t=1}^T
\inf_{h_t \in\mathcal{H}}
\frac{1}{m_t}
\sum_{i=1}^{m_t}
\ell\big(h_t\circ g (x_{t,i}),y_{t,i}\big).
$$

We are now ready to state our principal objective: we wish to design a procedure (meta-algorithm) that,
at the beginning of each task $t$, produces a function $\hat{g}_t$ so that,
within each task, the learner can use its own favorite online learning algorithm
to solve task $t$ on the sequence $\big((\hat{g}_t(x_{t,1}),y_{t,1}),\dots,(\hat{g}_t(x_{t,m_t}),y_{t,m_t})\big)$.  
We wish to control the {\em compound regret} of our procedure
\begin{equation*}
\mathcal{R} :=
\frac{1}{T}\sum_{t=1}^T \frac{1}{m_t} \sum_{i=1}^{m_t}\hat{\ell}_{t,i}
-
\inf_{g\in\mathcal{G}}
\frac{1}{T}\sum_{t=1}^T
\inf_{h_t \in\mathcal{H}} \frac{1}{m_t} \sum_{i=1}^{m_t}
   \ell\big(h_t\circ g (x_{t,i}),y_{t,i}\big)
\end{equation*}
which may succinctly be written as $\sup_{g\in\mathcal{G}} \big\{\frac{1}{T}\sum_{t=1}^T \mathcal{R}_t(g)\big\}$. 
This objective is accomplished in Section \ref{EWA-LL} under the assumption that a regret bound for the 
within-task-algorithm is available.

We end this section with two examples included in the framework.
%of the sets ${\cal G}$ and ${\cal H}$.
\begin{exm}
[Dictionary learning]
\label{exm-dic}
Set $\mathcal{Z}=\mathbb{R}^K$, and call $g=(g_1,\dots,g_K)$ a {\em dictionary}, where each $g_k$ is a real-valued function on ${\cal X}$. Furthermore choose 
${\cal H}$ to be a set of linear functions on $\mathbb{R}^K$, so that, for each task $t$
$$
h_t \circ g (x) = \sum_{k=1}^K \theta^{(t)}_k g_k(x). 
$$
In practice depending on the value of $K$, we can use
least square estimators or LASSO to learn $\theta^{(t)}$. In~\citep{maurer2013sparse,ruvolo2013ella}, 
the authors consider $\mathcal{X}=\mathbb{R}^d$ and $g(x) = Dx$ for some
$d\times K$ matrix $D$, and the goal is to learn jointly the predictors $\theta^{(t)}$ and the dictionary $D$.
\end{exm}

\begin{exm}
[Finite set $\mathcal{G}$]
\label{exm-rel}
We choose $\mathcal{G}=\{g_1,\dots,g_K\}$ and ${\cal H}$ any set. While this example is interesting in its own right, it is also 
instrumental in studying the continuous case via a suitable discretization process. 
A similar choice has been considered by \cite{Crammer} in the multitask setting, in which the goal is to bound 
the average error on a prescribed set of tasks.  

%is to use the $T$ tasks to learn a good representation which is likely to work well on future tasks (sampled by under the same process than the training tasks). This setting, which may be referred to as \textbf{batch-within-batch} or \textbf{online-within-batch}, is studied in Section~\ref{online-to-batch}.
% using the online-to-batch trick~\cite{cesa2006prediction}

\end{exm}

We notice that a slightly different learning setting is obtained when each dataset $ \mathcal{S}_t$ is given all at once. We refer to this as {\em batch-within-online} lifelong learning; this setting is briefly considered 
in Appendix~\ref{appendix-bwo}. On the other hand when all datasets are revealed all at once, we are in the well-known  setting 
of {\em learning-to-learn} \citep{Baxter2000}. In Section \ref{online-to-batch}, we explain how our lifelong learning analysis 
can be adapted to this setting.

%%%%%%%%%%%%%%%%%%%%%%%%%%%%%%%%

\begin{algorithm}
\caption{EWA-LL}
\begin{description}
\item[Data] A sequence of datasets
\\ $ \mathcal{S}_t = \big((x_{t,1},y_{t,1}),\dots,(x_{t,m_t},y_{t,m_t})\big) $,
$1\leq t\leq T$. associated with different learning tasks; the points within each dataset
are also given sequentially.
\item[Input] A prior $\pi_1$, a
learning parameter $\eta>0$ and a learning algorithm for each task $t$ which,
for any representation $g$ returns a
sequence of predictions $\hat{y}_{t,i}^g$ and suffers a loss
$$ \hat{L}_t(g) := \frac{1}{m_t} \sum_{i=1}^{m_t}
       \ell\left(\hat{y}_{t,i}^g,y_{t,i}\right). $$
\item[Loop] For $t=1,\dots,T$
\begin{description}
\item[i] Draw $\hat{g}_t \sim \pi_t$.
\item[ii] Run the within-task learning algorithm on $ \mathcal{S}_t $ and suffer
loss $ \hat{L}_t(\hat{g}_t)$.
\item[iii] Update
 $$
 \pi_{t+1}({\rm d}g) := \frac{\exp(-\eta \hat{L}_t(g))
                  \pi_{t}({\rm d}g) }{\int \exp(-\eta \hat{L}_t(\gamma))
                  \pi_{t}({\rm d}\gamma) }.
 $$
\end{description}
\end{description}
\label{Al: main}
\end{algorithm}

\section{ALGORITHM}
\label{EWA-LL}

In this section, we present our lifelong learning algorithm, derive its regret bound and then specify it to two popular within-task online algorithms.

\subsection{EWA-LL Algorithm}

Our EWA-LL algorithm is outlined in Algorithm \ref{Al: main}. The algorithm is based on the
exponentially weighted aggregation procedure \citep[see e.g.][and references therein]{cesa2006prediction}, and updates a probability distribution $\pi_t$ on the set of representation ${\cal G}
$ before the encounter of task $t$. The effect of Step {\bf iii} is that any representation $g$ which does not perform well on task $t$, is less likely to be reused on the next task. We insist on the fact that this procedure allows the user to freely choose the within-task algorithm, which does not even need to be the same for each task.

\subsection{Bounding the Expected Regret}

Since Algorithm \ref{Al: main} involves a randomization strategy, we can only get a bound on the
expected regret, the expectation being with respect to the
drawing of the function ${\hat g}_t$ at step {\bf i} in the algorithm. Let $\mathbb{E}_{g \sim \pi}[F(g)]$ denote
the expectation of $F(g)$ when $g\sim \pi$. Note that the expected
overall-average loss that we want to upper bound is then
$$
\frac{1}{T} \sum_{t=1}^T \mathbb{E}_{{\hat g}_t \sim \pi_t}[\hat{L}_t({\hat g}_t)].
$$

\begin{thm}
\label{thm:online:w:online}
If, for any $g \in {\cal G}$, $\hat{L}_t(g) \in [0,C]$ and the within-task algorithm has a regret bound $\mathcal{R}_t(g) \leq
 \beta(g,m_t)$, then
 \begin{multline*}
 \frac{1}{T} \sum_{t=1}^T \mathbb{E}_{{\hat g}_t \sim \pi_t}\left[
 \frac{1}{m_t}
 \sum_{i=1}^{m_t} \hat{\ell}_{t,i}
 \right]
 \leq
 \inf_{\rho} \Biggl\{
       \mathbb{E}_{g \sim \rho}\Biggl[
        \frac{1}{T}
       \sum_{t=1}^T
       \inf_{h_t\in\mathcal{H}}
       \frac{1}{m_t}
       \sum_{i=1}^{m_t}
        \ell\big(h_t\circ g(x_{t,i}),y_{t,i}\big)
      \\
      +  \frac{1}{T} \sum_{t=1}^T \beta(g,m_t)
       \Biggr]
     + \frac{\eta C^2}{8} + \frac{\mathcal{K}(\rho,\pi_1)}{\eta T}
 \Biggr\},
 \end{multline*}
 where the infimum is taken over all probability measures $\rho$ and $\mathcal{K}(\rho,\pi_1)$ is the Kullback-Leibler divergence between
$\rho$ and $\pi_1$.
\end{thm}
%MMM it would be nice to give some insights into the proof
The proof is given in Appendix~\ref{proofs}. 
Some comments are in order as the bound in Theorem~\ref{thm:online:w:online}
might not be easy to read. First, similar to standard analyses in online learning, the parameter $\eta$ is a decreasing function of $T$, hence the bound vanishes as $T$ grows. Second, corollaries are derived
in Section~\ref{applications} that are easier to read, as they are more
similar to usual regret inequalities \citep{cesa2006prediction}, that is, the right hand side of the bound is of the form
\begin{equation}
\label{oracle-type}
\hspace{-.2truecm}\inf_{g\in\mathcal{G}}
  \frac{1}{T}
       \hspace{-.03truecm}\sum_{t=1}^T
       \hspace{-.03truecm} \inf_{h_t\in\mathcal{H}}\hspace{-.03truecm}
       \frac{1}{m_t}\hspace{-.03truecm}
       \sum_{i=1}^{m_t}
        \ell\big(h_t\circ g(x_{t,i}),y_{t,i}\big) + \text{``rate''}.
\end{equation}
The bound in Theorem \ref{Al: main} looks slightly different, but is quite similar in spirit.
Indeed, instead of an infimum with respect to $g$ we have an infimum on all
the possible aggregations with respect to $g$,
\begin{equation*}
 \inf_{\rho}
       \mathbb{E}_{g \sim \rho}
        \frac{1}{T}
       \sum_{t=1}^T
       \inf_{h_t\in\mathcal{H}}
       \frac{1}{m_t}
       \sum_{i=1}^{m_t}
        \ell\big(h_t\circ g(x_{t,i}),y_{t,i}\big)
+ \text{``remainder''}
\end{equation*}
where the remainder term depends on $\mathcal{K}(\rho,\pi_1)$. In order to look
like~\eqref{oracle-type}, we could consider a measure $\rho$ highly concentrated
around the representation $g$ minimizing~\eqref{oracle-type}. When $\mathcal{G}$ is finite, this is a reasonable strategy and the bound is
given explicitly in Section~\ref{subsection-finite} below. 
However, in some situations, this would cause the term
$\mathcal{K}(\rho,\pi_1)$ to diverge. 
%explode. 
Studying accurately the minimizer in $\rho$ usually leads to an 
interesting regret bound, and this is exactly what is done in Section \ref{applications}.

Finally note that the bound in Theorem~\ref{thm:online:w:online} is given
in expectation. In online learning, uniform bounds are usually
prefered~\citep{cesa2006prediction}. In Section~\ref{uniform} we show that
it is possible to derive such bounds under additional assumptions.

\subsection{Examples of Within Task Algorithms}

We now specify the general bound in Theorem \ref{Al: main} to two popular online algorithms which we use within tasks.
\subsubsection{Online Gradient}
\label{exm:grad}

The first algorithm assumes that $\mathcal{H}$ is a parametric family of functions
$\mathcal{H}=\{h_\theta,\theta \in\mathbb{R}^p, \|\theta\|\leq B \}$, and 
for any $(x,y,g)$, $\theta\mapsto \ell(h_{\theta}\circ g(x),y)$
is convex, $L$-Lipschitz, upper bounded by $C$ and denote by $\nabla_\theta$ a subgradient. 
\begin{algorithm}[t]
\caption{OGA}
\label{OGA}
\begin{description}
\item[Data] A task
$ \mathcal{S}_t = \big((x_{t,1},y_{t,1}),\dots,(x_{t,m_t},y_{t,m_t})\big)$.
\item[Input] Stepsize $\zeta>0$, and $\theta_1 = 0$.
\item[Loop] For $i=1,\dots,m_t$,
\begin{description}
\item[i] Predict $\hat{y}^g_{t,i}= h_{\theta_{i}}\circ g (x_{t,i})$,
\item[ii] $y_{t,i}$ is revealed, update 
\\$\theta_{i+1}=\theta_i-\zeta \nabla_\theta
 \ell \big(h_{\theta}\circ g (x_{t,i}),y_{t,i}\big) \big|_{\theta=\theta_i} $.
\end{description}
\end{description}
\end{algorithm}

\begin{cor}
\label{cor:1}
The EWA-LL algorithm using the OGA within task with step size $\zeta=\frac{B}{L\sqrt{2m_t}}$ satisfies
\begin{multline*}
 \frac{1}{T} \sum_{t=1}^T \mathbb{E}_{{\hat g}_t \sim \pi_t}\left[
 \frac{1}{m_t}
 \sum_{i=1}^{m_t} \hat{\ell}_{t,i}
 \right]
 \leq
 \inf_{\rho} \Biggl\{
       \mathbb{E}_{g \sim \rho}\Biggl[
        \frac{1}{T}
       \sum_{t=1}^T
       \inf_{h_t\in\mathcal{H}}
       \frac{1}{m_t}
       \sum_{i=1}^{m_t}
        \ell(h_t\circ g(x_{t,i}),y_{t,i})
      \\
      +  \frac{BL}{T} \sum_{t=1}^T \sqrt{\frac{2}{m_t}}
       \Biggr]
     + \frac{\eta C^2}{8} + \frac{\mathcal{K}(\rho,\pi_1)}{\eta T}
 \Biggr\}.
 \end{multline*}
\end{cor}

\begin{proof}
 Apply Theorem~\ref{thm:online:w:online} and use the bound
 $\mathcal{R}_t(g) \leq \beta(g,m_t) := BL\sqrt{2/ m_t}$ that can be found,
 for example, in \cite[Corollary 2.7]{shalev2011online}.
\end{proof}
%MMM MASSI commented this 
We note that under additional assumptions on loss functions,~\cite[Theorem 1]{hazan2007logarithmic} provides bounds for $\beta(g,m_t)$ that are in $\log(m_t)/m_t$. 

\subsubsection{Exponentially Weighted Aggregation}
\label{exm:ewa}
The second algorithm is based on the EWA procedure on the space ${\cal H} \circ g$ for a prescribed 
representation $g \in {\cal G}$. 
\begin{algorithm}
\caption{EWA}
\begin{description}
\item[Data] A task
$ \mathcal{S}_t = \big((x_{t,1},y_{t,1}),\dots,(x_{t,m_t},y_{t,m_t})\big) $.
\item[Input] Learning rate $\zeta>0$; a prior 
probability distribution $\mu_1$ on $\mathcal{H}$.
\item[Loop] For $i=1,\dots,m_t$,
\begin{description}
\item[i] Predict $\hat{y}_{t,i}^g= \int_{\mathcal{H}}
 h \circ g (x_{t,i}) \mu({\rm d} h)$,
\item[ii] $y_{t,i}$ is revealed, update
$$ \mu_{i+1}({\rm d}h) = \frac{\exp(-\zeta \ell(h \circ g (x_{t,i}),y_{t,i}))
    \mu_{i}({\rm d}h)}{\int \exp(-\zeta \ell(u \circ g (x_{t,i}),y_{t,i}))
    \mu_{i}({\rm d}u)}. $$
\end{description}
\end{description}
\end{algorithm}
Recall that a function $\varphi: {\mathbb R} \rightarrow {\mathbb R}$ is called $\zeta_0$-exp-concave 
if $\exp(-\zeta_0 \varphi)$ is concave.
\begin{cor}
\label{cor:2}
Assume that $\mathcal{H}$ is finite and that there exists $\zeta_0>0$ such that for any $y$, 
the function $\ell(\cdot,y)$ is $\zeta_0$-exp-concave and upper bounded by a constant $C$. 
Then the EWA-LL algorithm using the EWA within task with $\zeta= \zeta_0$ satisfies
 \begin{multline*}
 \frac{1}{T} \sum_{t=1}^T \mathbb{E}_{{\hat g}_t \sim \pi_t}\left[
 \frac{1}{m_t}
 \sum_{i=1}^{m_t} \hat{\ell}_{t,i}
 \right]
 \leq
 \inf_{\rho} \Biggl\{
       \mathbb{E}_{g \sim \rho}\Biggl[
        \frac{1}{T}
       \sum_{t=1}^T
       \inf_{h_t\in\mathcal{H}}
       \frac{1}{m_t}
       \sum_{i=1}^{m_t}
        \ell\big(h_t\circ g(x_{t,i}),y_{t,i}\big)
      \\
      +   \frac{1}{T} \sum_{t=1}^T \frac{ \zeta_0 \log|\mathcal{H}|}{m_t}
       \Biggr]
     + \frac{\eta C^2}{8} + \frac{\mathcal{K}(\rho,\pi_1)}{\eta T}
 \Biggr\}.
 \end{multline*}
\end{cor}

\begin{proof}
  Apply Theorem~\ref{thm:online:w:online} and use the bound
 $\mathcal{R}_t(g) \leq \beta(g,m_t) := \zeta_0 \log|\mathcal{H}| /m_t$ that can be found,
 for example, in \cite[Theorem 2.2]{gerchinovitz2011prediction}.
\end{proof}

A typical example is the quadratic loss function $\ell(y',y)=(y'-y)^2$. When there is some $B$ such that 
$|y_{t,i}|\leq B$ and $|h\circ g(x_{t,i})|\leq B$, then the exp-concavity assumption is verified 
with $\zeta_0= 1/(8B)$ and the boundedness assumption with $C = 4B^2$.

Note that when
the exp-concavity assumption does not hold,~\cite{gerchinovitz2011prediction} derives a bound $\beta(g,m_t)=
B \sqrt{\log (|\mathcal{H}|) /(2m_t)}$ with
$\zeta=(2/B) \sqrt{2 \log (|\mathcal{H}|) / m_t}$. 
Moreover, PAC-Bayesian type bounds in various settings (including infinite $\mathcal{H}$) can be found in~\citep{catoni2004statistical,audibert2006randomized,gerchinovitz2013sparsity}. 
We refer the reader to~\citep{gerchinovitz2011prediction} for a comprehensive survey.

%%%%%%%%%%%%%%%%%%%%%%%%%%%%%%%%

\section{APPLICATIONS}
\label{applications}
In this section, we discuss two important applications. To ease our presentation, we assume that all the tasks have the same sample size, that is $m_t=m$ for all $t$.

\subsection{Finite Subset of Relevant Predictors}
\label{subsection-finite}
\label{subsec: finite}

We give details on Example~\ref{exm-rel}, that is we assume that $\mathcal{G}$ 
is a set of $K$ functions. Note that step {\bf iii} in Algorithm \ref{Al: main} boils down to update $K$ weights, 
$$
\pi_t (g_k) = \frac{\exp(- \eta \hat{L}_t (g_k) ) \pi_{t-1}(g_k) }
                 {\sum_{j=1}^K \exp(- \eta \hat{L}_t (g_j) ) \pi_{t-1}(g_j) }.
$$

\begin{thm}
 Under the assumptions of Theorem \ref{thm:online:w:online}, if we set $\eta
 =\frac{2}{C}
 \sqrt{\frac{2 \log K}{T}}$ and $\pi_1$ uniform on $\mathcal{G}$,
 \begin{multline*}
\frac{1}{T} \sum_{t=1}^T \mathbb{E}_{{\hat g}_t \sim \pi_t}\left[
   \frac{1}{m} \sum_{i=1}^{m} \hat{\ell}_{t,i}
 \right]
 \leq
       \min_{1\leq k\leq K} \Bigg\{
       \frac{1}{T}\sum_{t=1}^T
       \inf_{h_t\in\mathcal{H}}
        \frac{1}{m} \sum_{i=1}^{m}
        \ell(h_t\circ g_k(x_{t,i}),y_{t,i})
 \\      +\beta(g_k,m)
       \Bigg\}
       + C\sqrt{\frac{\log K}{2T}}.
 \label{example1}
 \end{multline*}
\end{thm}

\begin{proof}
Fix $g\in\mathcal{G}$, $\rho$ as the
Dirac mass on $g$ and note that $\mathcal{K}(\rho,\pi_1)=\log K$.
\end{proof}

We discussed in Sections~\ref{exm:grad} and~\ref{exm:ewa} that typical orders for 
$\beta(g,m)$ are $\mathcal{O}(1/\sqrt{m})$, $\mathcal{O}(\log(m)/m)$ or $\mathcal{O}(1/m)$. We state a precise result in the finite case.

\begin{cor}
Assume that $\mathcal{H}$ is finite, that
for some $\zeta_0>0$, for any $y$, the function $\ell(\cdot,y)$ is $\zeta_0$-exp-concave and upper bounded by a constant $C$. Then the EWA-LL algorithm using the EWA within task with $\zeta= \zeta_0$ satisfies
 \begin{multline*}
 \frac{1}{T} \sum_{t=1}^T \mathbb{E}_{{\hat g}_t \sim \pi_t}\left[
 \frac{1}{m}
 \sum_{i=1}^{m} \hat{\ell}_{t,i}
 \right]
 \leq
 \min_{1\leq k \leq K}
        \frac{1}{T}
       \sum_{t=1}^T
       \min_{h_t\in\mathcal{H}}
       \frac{1}{m}
       \sum_{i=1}^{m}
        \ell(h_t\circ g_k(x_{t,i}),y_{t,i})
      \\
      + \frac{ \zeta_0 \log|\mathcal{H}|}{m}
     + C\sqrt{\frac{\log K}{2T}}.
 \end{multline*}
\end{cor}

In Section~\ref{online-to-batch}, we derive from Theorem~\ref{thm:online:w:online}
a bound in the batch setting. As we shall see, in the finite case the bound is exactly the same as the bound on the compound regret.
This allows us to compare our results to previous ones obtained in the learning-to-learn
setting. In particular, our $\mathcal{O}(1/m)$ bound improves upon \citep{pentina2014pac} who derived an $\mathcal{O}(1/\sqrt{m})$ bound.
%, whereas our bound is $\mathcal{O}(1/m)$.

\subsection{Dictionary Learning}

We now give details on Example~\ref{exm-dic} in the 
linear case. Specifically, we let $\mathcal{X}=\mathbb{R}^d$, we let $\mathcal{D}_K$ be the set
formed by all $d\times K$ matrices $D$, whose columns have euclidean norm equal to one, and we define $\mathcal{G}= \{ x\mapsto D x : D \in \mathcal{D}_K \}$.
Within this subsection we assume that the loss $\ell$ is convex and $\Phi$-Lipschitz
with respect to its first argument, that is,
for every $y \in {\cal Y}$ and $a_1,a_2 \in {\mathbb R}$, it holds $|\ell(a_1,y) -\ell(a_2,y)| \leq \Phi |a_1-a_2 |$. 
We also assume that for all $t\in \{1,\dots,T\}$ and $i\in \{1,\dots,m\}$, $ \| x_{t,i}\| \leq 1 $. Assume $\beta(m):=\sup_{g\in\mathcal{G}} \beta(m,g)<+\infty$.

We define the prior $\pi_1$ as follows: the columns of $D$ are
i.i.d., uniformly distributed on the $d$-dimensional unit sphere.

\begin{thm}
\label{coro-dico}
 Under the assumptions of Theorem \ref{thm:online:w:online}, with
 $ \eta =\frac{2}{C} \sqrt{\frac{Kd}{T}} $,
\begin{multline*}
 \frac{1}{T} \sum_{t=1}^T \mathbb{E}_{{\hat g}_t \sim
  \pi_t}\left[ \frac{1}{m} \sum_{i=1}^{m} 
  \hat{\ell}_{t,i} \right]
 \leq
    \inf_{D\in\mathcal{D}_K} \Biggl\{ \frac{1}{T}
\sum_{t=1}^T\inf_{h_t \in\mathcal{H}}
\frac{1}{m}\sum_{i=1}^{m}
 \ell\big(\langle h_t, D x_{t,i}\rangle,y_{t,i}\big)
 \\
 +
\frac{C}{4} \sqrt{\frac{Kd}{T}}( \log(T)+7)
+ \beta(m) \Biggr\}
+ \frac{B\Phi}{\sqrt{T}} \sqrt{ \frac{1}{T}       \sum_{t=1}^T
   \lambda_{\max} \left( \frac{1}{m}       \sum_{i=1}^{m}
   x_{t,i} x_{t,i}^T \right)
  }.
\end{multline*}
\end{thm}
The proof relies on an application of Theorem~\ref{thm:online:w:online}.
The calculations being tedious, we postpone the proof to Appendix~\ref{proofs}.

When we use OGA within tasks, we can use Corollary~\ref{cor:1}
with $L=\Phi \sqrt{K}$ and 
so $ \beta(m) \leq \Phi B
\sqrt{2K/m} $ for any $D\in\mathcal{D}_K$. Moreover,
\begin{equation}
\label{equa-stupid}
\hspace{-.3truecm}\lambda_{\max} \left( \hspace{-.071truecm}\frac{1}{m}       \sum_{i=1}^{m}
 x_{t,i} x_{t,i}^T \hspace{-.071truecm}\right)
\leq {\rm tr} \left(\hspace{-.071truecm} \frac{1}{m}       \sum_{i=1}^{m}
 x_{t,i} x_{t,i}^T \hspace{-.071truecm}\right ) \leq 1
\end{equation}
so Theorem \ref{coro-dico} leads to the following corollary. 

\begin{cor}
 Algorithm EWA-LL for dictionary learning, with
 $ \eta =(2/C)\sqrt{Kd/T} $, and using the OGA algorithm within tasks, with step $\zeta = B/(\Phi\sqrt{2mK})$, satisfies
\begin{multline*}
 \frac{1}{T} \sum_{t=1}^T \mathbb{E}_{{\hat g}_t \sim
  \pi_t}\left[ \frac{1}{m} \sum_{i=1}^{m} 
  \hat{\ell}_{t,i} \right]
  \leq 
   \inf_{D\in\mathcal{D}_K}\frac{1}{T}
\sum_{t=1}^T\inf_{h_t \in\mathcal{H}}
\frac{1}{m}\sum_{i=1}^{m}
 \ell\big(\langle h_t, D x_{t,i}\rangle,y_{t,i}\big)
 \\
 +
\frac{C}{4} \sqrt{\frac{Kd}{T}}( \log(T)+7)
+ \frac{B\Phi }{\sqrt{T}}
+ \frac{\Phi  B\sqrt{2 K}}{\sqrt{m}} .
\end{multline*}
\end{cor}

Note that the upper bound~\eqref{equa-stupid} may be lose. For example, when the $x_{t,i}$ are i.i.d. on the unit sphere, $
\lambda_{\max} \left( \sum_{i=1}^{m}
 x_{t,i} x_{t,i}^T /m \right) $ is close to $1/d $.
 In this case, it is possible to improve the term $\beta(m)$ employed in the calculation of the bound, we postpone the lengthy details to Appendix~\ref{appendix:improvement}.

\subsubsection{Algorithmic Details and Simulations}

We implement our meta-algorithm Randomized EWA in 
this setting. The algorithm used within each task is the 
simple version of the online gradient algorithm outlined in 
Section~\ref{exm:grad}.

In order to draw $\hat{g}_t$ from $\pi_t$, we use $N$-steps of
Metropolis-Hastings algorithm with a normalized 
Gaussian proposal \citep[see, for example,][]{robert2013monte}. 
In order to ensure a short burn-in 
period, we use the previous drawing $\hat{g}_{t-1}$ as
a starting point. The procedure is given in Algorithm~\ref{EWA-LL-d}.
\begin{algorithm}[t]
\caption{EWA-LL for dictionary learning}
\label{EWA-LL-d}
\begin{description}
\item[Data]  As in Algorithm \ref{Al: main}.
\item[Input] A learning rate $\eta$ for EWA and a 
learning rate $\zeta$ for the online gradient. A number 
of steps $N$ for the Metropolis-Hastings algorithm.
\item[Start] Draw 
$\hat{g}_1\sim\pi_1$.
\item[Loop] For $t=1,\dots,T$
\begin{description}
\item[i] Run the within-task learning algorithm $ \mathcal{S}_t $ and suffer
loss $ \hat{L}_t(\hat{g}_t)$.
\item[ii] Set $\tilde{g} := \hat{g}_t$.
\item[iii] Metropolis-Hastings algorithm. Repeat $N$ times \begin{description}
 \item[a] Draw $\tilde{g}' \sim \mathcal{N}(\tilde{g},\sigma^2 I)$ and
 then set $\tilde{g}' := \tilde{g}'/\|\tilde{g}'\|$.
 \item[b] Set $
 \tilde{g} := \tilde{g}'$ with probability
 $$
 \min\left\{1,\exp\left[\eta \sum_{h=1}^t \left( \hat{L}_h (\tilde{g}) -
        \hat{L}_h (\tilde{g}') \right) \right] \right\},
 $$
 $\tilde{g}$ remains unchanged otherwise.
\end{description}
\item[iv] Set $\hat{g}_t := \tilde{g}$.
\end{description}
\end{description}
\end{algorithm}
Note the bottleneck of the algorithm: in step {\bf b} we have to compare
$\tilde{g}$ and $\tilde{g}'$ on the whole dataset so far.

We now present a short simulation study. We generate data in the following
way: we let $K=2$, $d=5$, $T=150$ and $m=100$. The columns
of $D$ are drawn uniformly on the unit sphere,
and task regression vectors $\theta_t$ are also independent and
have i.i.d. coordinates in $\mathcal{U}[-1,1]$.
We generate the datasets $\mathcal{S}_t$ as follows:
all the $x_{t,i}$ are i.i.d. from the same distribution as $\theta_t$,
and
$ y_{t,i} = \langle\theta_t,D x_{t,i} \rangle + \varepsilon_{t,i} $
where the $\varepsilon_{t,i}$ are i.i.d. $\mathcal{N}(0,\sigma^2)$ and
$\sigma=0.1$.

We compare Algorithm~\ref{EWA-LL-d} with $N=10$ to an oracle who knows the representation $D$, but not the task regression vectors $\theta_t$, and learns them using the online gradient algorithm with step size $\zeta=0.1$. Notice that
after each chunk of $100$ observations, a new task starts, so the parameter $\theta_t$ changes. Thus, the oracle incurs
a large loss until it learns the new $\theta_t$ (usually within a few steps). This
explains the ``stair'' shape of the cumulative loss of the oracle in
Figure~\ref{figure-zoom}. Figure~\ref{valeur-borne1} indicates that
after a few tasks, the dictionary $D$ is learnt by EWA-LL: its cumulative loss becomes parallel to the one of the oracle.
Due to the bottleneck mentioned above, the algorithm becomes quite slow to run
when $t$ grows. In order to improve the speed of the algorithm, we also tried Algorithm  \ref{EWA-LL-d} with $N=1$. There is absolutely no theoretical justification for
this, however, obviously the algorithm is 10 times faster. As we can see
on the red line in Figure~\ref{valeur-borne1}, this version of the algorithm
still learns $D$, but it takes more steps.
Note that this is not completely
unexpected: the Markov chain generated by this algorithm is no longer
stationary, but it can still enjoy good mixing properties. It would be
interesting to study the theoretical performance of Algorithm \ref{EWA-LL-d}. However,
this would require considerably technical tools from Markov chain theory which are beyond the scope of this paper.
\begin{figure}[t]
\begin{center}
\includegraphics[scale=0.5]{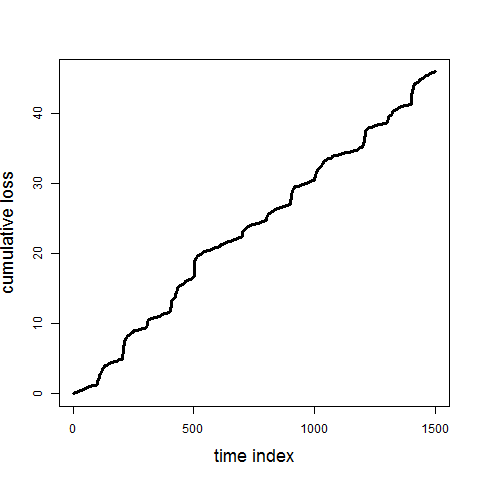}
\end{center}
\caption{The cumulative loss of the oracle for the first $15$ tasks.}
\label{figure-zoom}
\end{figure}
\begin{figure}[t]
\begin{center}
\includegraphics[scale=0.5]{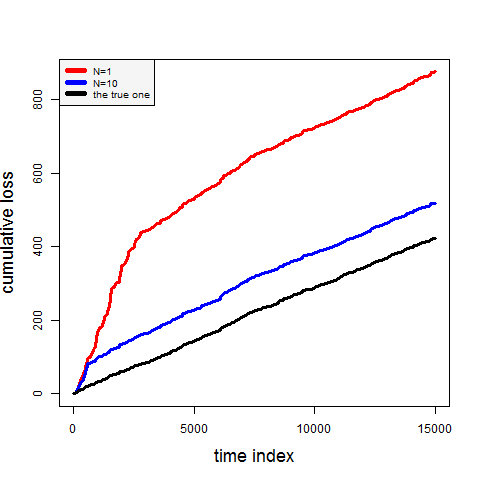}
\end{center}
\caption{Cumulative loss of EWA-LL ($N=1$ in red and
$N=10$ in blue) and cumulative loss of the oracle.}
\label{valeur-borne1}
\end{figure}

\section{UNIFORM BOUNDS}
\label{uniform}
In this section, we show that it possible to obtain a 
uniform bound, as opposed to a bound in expectation 
as in Theorem~\ref{thm:online:w:online}. From a theoretical perspective, the price to pay is very low: we only have to
assume that the loss function is convex with respect to its first
argument. However, in practice, there is an aggregation step
that might not be feasible. This is discussed
at the end of the section. The algorithm is outlined in Algorithm \ref{alg:04}.
\begin{thm}
\label{cor:online:w:online:agg}
Assuming that for any $g,0\leq \hat{L}_t(g) \leq C $ 
and that the algorithm used within-task has a regret $\mathcal{R}_t(g) \leq \beta(g,m_t)$. Assume that 
$\ell$ is convex with respect to its first argument. Then it holds that
 \begin{multline*}
 \frac{1}{T} \sum_{t=1}^T \frac{1}{m_t} \sum_{i=1}^{m_t}
       \ell\left(\hat{y}_{t,i},y_{t,i}\right)
 \leq
 \inf_{\rho} \Bigg\{
       \mathbb{E}_{g \sim \rho}\left[
        \frac{1}{T}   \sum_{t=1}^T
       \inf_{h_t\in\mathcal{H}}
       \frac{1}{m_t}       \sum_{i=1}^{m_t}
        \ell(h_t\circ g(x_{t,i}),y_{t,i})\right.
  \\   
 \left.  +  \frac{1}{T} \sum_{t=1}^T \beta(g,m_t) \right]
     + \frac{\eta C^2}{8} + \frac{\mathcal{K}(\rho,\pi_1)}{\eta T}
 \Bigg\}.
 \end{multline*}
\end{thm}
\begin{proof}
At each step $t$, the loss suffered by the algorithm is
\begin{multline*}
\frac{1}{m_t} \sum_{i=1}^{m_t}
       \ell\big(\hat{y}_{t,i},y_{t,i}\big) =
\frac{1}{m_t} \sum_{i=1}^{m_t}
       \ell\bigg(\int  \hat{y}_{t,i}^{g} \pi_t({\rm d}g),y_{t,i}\bigg)
 \\
 \leq
\frac{1}{m_t} \sum_{i=1}^{m_t} \int
       \ell\big(\hat{y}_{t,i}^{g},y_{t,i}\big) \pi_t({\rm d} g)
= \int \hat{L}_t(g) \pi_t({\rm d}g)
\end{multline*}
and we can just apply Theorem~\ref{thm:online:w:online}.
\end{proof}

\begin{algorithm}[t]
\caption{Integrated EWA-LL}
\begin{description}
\item[Data and Input] same as in Algorithm \ref{Al: main}.
\item[Loop] For $t=1,\dots,T$
\begin{description}
\item[i] Run the within-task learning algorithm on $ \mathcal{S}_t $ for each
$g\in\mathcal{G}$
and return as predictions:
\begin{equation}
\label{integral}
\hat{y}_{t,i}
 = \int  \hat{y}_{t,i}^{g} \pi_t({\rm d}g).
\end{equation}
\item[ii] Update
 $
 \pi_{t+1}({\rm d}g) := \frac{\exp(-\eta \hat{L}_t(g))
                  \pi_{t}({\rm d}g) }{\int \exp(-\eta \hat{L}_t(\gamma))
                  \pi_{t}({\rm d}\gamma) }.
 $
\end{description}
\end{description}
\label{alg:04}
\end{algorithm}
In practice, for an infinite set $\mathcal{G}$ we are not able to run
simultaneously the within-task algorithm for all $g\in\mathcal{G}$. So,
we cannot compute the prediction~\eqref{integral} exactly. A possible
strategy is to draw $N$ elements of $\mathcal{G}$ i.i.d. from $\pi_t$,
say $\hat{g}_{t}(1),\dots,\hat{g}_{t}(N)$, and to replace~\eqref{integral}
by
$$ \hat{y}_{t,i}^{(N)}
 = \frac{1}{N} \sum_{j=1}^N \hat{y}_{t,i}^{\hat{g}_t(j)}. $$
An application of Hoeffding's inequality shows for any
$\delta>0$, with probability at least $1-\delta$, the bound
in Theorem~\ref{cor:online:w:online:agg} will still hold,
up to an additional term $C\sqrt{\log(T/\delta)/2N}$.

\section{LEARNING-TO-LEARN}
\label{online-to-batch}

In this section, we show how our analysis of lifelong learning can be used to derive bounds for learning-to-learn. In this setting, the tasks and their datasets are generated by first sampling task distributions $P_1,\dots,P_T$ i.i.d. from a``meta-distribution''$Q$, called {\em environment} by \cite{Baxter2000}, and then for each task $t$, a dataset ${\cal S}_t = ((x_{t,1},y_{t,1}),\dots,(x_{t,m},y_{t,m}))$ is sampled i.i.d. from $P_t$. We stress that in this setting, the entire data $(x_{t,i},y_{t,i})_{1\leq i \leq m,1\leq t\leq T}$ is given all at once to the learner. Note that for simplicity, we assumed that all the sample sizes are the same.
%, although it is easy to extend the setting to a random $m_t$ drawn at each step.}.

We wish to design a strategy which, given a new task $P\sim Q$ and a new sample $(x_1,y_1),\dots,(x_m,y_m)$
i.i.d. from $P$, computes a function $f:\mathcal{X}\rightarrow \mathcal{Y}$, that will predict $y$ well when $(x,y)\sim P$. For this purpose we propose the following strategy:
\begin{enumerate}
 \item Run EWA-LL on $(x_{t,i},y_{t,i})_{1\leq i \leq m,1\leq t\leq T}$. We obtain a sequence of representations 
 $\hat{g}_1,\dots,\hat{g}_T$,
 \vspace{-.1truecm}
 \item Draw uniformly $\mathcal{T}\in\{1,\dots,T\}$ and put $\hat{g} =
 \hat{g}_{\mathcal{T}}$,
 \vspace{-.1truecm}
 \item Run the within task algorithm on the sample
  $(x_i,y_i)_{1\leq i\leq m}$, obtaining a sequence
  $h_{1}^{\hat{g}},\dots,h_{m}^{\hat{g}}$ of functions,
  \vspace{-.1truecm}
 \item Draw uniformly $\mathcal{I}\in\{1,\dots,m\}$ and put $\hat{h} =
 h_{\mathcal{I}}^{\hat{g}} $.
\end{enumerate}
Our next result establishes that the strategy leads indeed to safe predictions.
\begin{thm}
\label{thm:online:to:batch}
Let $\mathbb{E}$ be the expectation over all data pairs $(x_{t,i},y_{t,i})_{1\leq i\leq m}\sim P_t$,
$(P_t)_{1\leq t\leq T}\sim Q$,
$(x_i,y_i)_{1\leq i \leq m} \sim P$, $(x,y)\sim P$, $P\sim Q$
and also over the randomized decisions of the learner
$(\hat{g}_t)_{1\leq t\leq T}$,
$\mathcal{T}$ and $\mathcal{I}$.
Then
\begin{multline*}
 \mathbb{E} [\ell(\hat{h}\circ \hat{g}(x),y)]
 \leq
 \inf_{\rho}
 \Biggl\{
 \mathbb{E}_{g\sim \rho}
 \Biggl[
 \mathbb{E}_{P\sim Q} \inf_{h\in\mathcal{H}}
 \mathbb{E}_{(x,y)\sim P} \Bigl[\ell(h\circ g(x),y)\Bigr]
\\ + \beta(g,m)
 \Biggr]
 + \frac{\eta C^2 }{8} + \frac{\mathcal{K}(\rho,\pi_1)}{\eta T}
 \Biggr\}.
\end{multline*}
\end{thm}
The proof is given in Appendix~\ref{proofs}. As in Theorem~\ref{thm:online:w:online},
the result is given in expectation with respect to the randomized decisions
of the learner. Assuming that $\ell$ is convex with respect to its first argument,
we can state a similar result for a non-random procedure, as was done in
Section~\ref{uniform}. Details are left to the reader.

\begin{rmk}
In~\citep{Baxter2000,maurer2013sparse,pentina2014pac}, the results on learning-to-learn are given with large probability with respect to the observations $(x_{t,i},y_{t,i})_{1\leq i\leq m,1\leq t\leq T}$, rather than in expectation. Using the machinery in \citep[Lemma 4.1]{cesa2006prediction} we conjecture that it is possible to derive a bound in probability from Theorem~\ref{thm:online:to:batch}.
\end{rmk}

%%%%%%%%%%%%%%%%%%%%%%%%%%%%%%

\section{CONCLUDING REMARKS}
\label{conclusion}

We presented a meta-algorithm for lifelong learning and derived 
%for the first time 
a fully online analysis of its regret. An important advantage of this algorithm is that it inherits the good properties of any algorithm used to learn within tasks. Furthermore, using online-to-batch conversion techniques, we derived bounds for the related framework of learning-to-learn. 

We discussed the implications of our general regret bounds for two applications: dictionary learning and finite set $\mathcal{G}$ of representations. Further applications of this algorithm which may be studied within our framework are deep neural networks and kernel learning. In the latter case, which has been addressed by ~\cite{PenBenDavid15} in the learning-to-learn setting, $g:\mathcal{X}\rightarrow\mathcal{Z}$ is a feature map to a
reproducing kernel Hilbert space $\mathcal{Z}$, and $h_t(g(x)) = \langle z^{(t)},g(x) \rangle_\mathcal{Z}$. In the former case, $\mathcal{X} = \mathbb{R}^d$ and $g: \mathcal{X} \rightarrow \mathbb{R}^K$ 
is a multilayer network, that is a vector-valued function obtained by 
application of a linear transformation and a nonlinear activation function. 
The predictor $h : \mathbb{R}^K \rightarrow \mathbb{R}$ is typically a linear 
function. The vector-valued function $(h_1 \circ  g,\dots,h_T \circ g))$ models 
a multilayer network with shared hidden weights. This is
discussed in~\citep{MPR2016}, again in the  learning-to-learn setting. 

Perhaps the most fundamental question is to extend our analysis to more computationally efficient algorithms such as approximations of EWA, like Algorithm~\ref{EWA-LL-d}, or fully gradient based algorithms as in~\citep{ruvolo2013ella}.

%%%%%%%%%%%%%%%%%%%%%%%%%%%%%%
%%%%%%%%%%%%%%%%%%%%%%%%%%%%%%

%\subsubsection*{Acknowledgements}

%%%%%%%%%%%%%%%%%%%%%%%%%%%%%%
%%%%%%%%%%%%%%%%%%%%%%%%%%%%%%
\addcontentsline{toc}{section}{References}

\renewcommand{\refname}{References}

%%%%%%%%%%%%%%%%%%%%%%%%%%%%%%
%%%%%%%%%%%%%%%%%%%%%%%%%%%%%%

\newpage

\appendix

\section{Proofs}
\label{proofs}

\begin{proof}[Proof of Theorem~\ref{thm:online:w:online}]
It is enough to show that the EWA strategy leads to
\begin{equation}
\label{lemma-pacbayes}
 \sum_{t=1}^T \mathbb{E}_{{\hat g}_t \sim \pi_t}[\hat{L}_t({\hat g}_t)]
 \leq
 \inf_{\rho} \Biggl\{
       \mathbb{E}_{g \sim \rho}\left[ \sum_{t=1}^T \hat{L}_t(g)\right]
       + \frac{\eta C^2 T}{8} + \frac{\mathcal{K}(\rho,\pi_1)}{\eta}
 \Biggr\}.
\end{equation}
Once this is done, we only have to use the assumption that the regret
of the within-task algorithm on task $t$ is upper bounded by $\beta(g,m_t)$ to obtain that
$$
\sum_{t=1}^T \hat{L}_t(g)
 = \sum_{t=1}^T \frac{1}{m_t} \sum_{i=1}^{m_t} \ell\big(
h_{t,i}^g \circ g(x_{t,i}),y_{t,i} \big)
\leq \sum_{t=1}^T \Biggl\{\beta(g,m_t)
+ \inf_{h\in\mathcal{H}} \frac{1}{m_t} \sum_{i=1}^{m_t} \ell\big(
h \circ g(x_{t,i}),y_{t,i} \big) \Biggr\}
$$
and we obtain the statement of the result. 

It remains to prove~\eqref{lemma-pacbayes}. To this end, we follows the same guidelines as in the proof of Theorem 1
in~\citep{audibert2006randomized}. First, note that
\begin{equation}
\label{renorm}
\pi_t(g) = \frac{\exp\left[-\eta \sum_{u=1}^{t-1} \hat{L}_u (g) \right]
         \pi_1({\rm d}g) }{ \int \exp\left[-\eta \sum_{u=1}^{t-1} \hat{L}_u
         (\gamma) \right] \pi_1({\rm d}\gamma)} 
  = \frac{\exp\left[-\eta \sum_{u=1}^{t-1} \hat{L}_u (g) \right]
         \pi_1({\rm d}g) }{W_t}
\end{equation}
where we introduce the notation $W_t$ for the sake of shortness.
Put $E_t = \int \hat{L}_t (g) \pi_t({\rm d} g)
= \mathbb{E}_{{\hat g}_t \sim \pi_t} [\hat{L}_t (g)]$.
Using Hoeffding's inequality on the bounded random
variable $\hat{L}_t (g)\in[0,C]$ we have, for any $t$, that
$$
 \mathbb{E}_{{\hat g}_t \sim \pi_t} \left[  \exp\left\{ 
   \eta (E_t-\hat{L}_t(g))
   \right\} \right] =
  \int 
 \exp\left\{ 
   \eta (E_t-\hat{L}_t(g))
   \right\} \pi_t({\rm d} g) \leq \exp\left\{ \frac{ C^2 \eta^2}{8}
 \right\}
$$
which can be rewritten as
\begin{equation}
 \label{hoeffding}
\exp\left\{ - \eta \mathbb{E}_{g_t \sim 
\pi_t}[\hat{L}_t(g_t)]
 \right\} 
   \geq 
   \exp\left( -\frac{C^2 \eta^2}{8} \right)
 \mathbb{E}_{{\hat g}_t \sim \pi_t} \left\{ \exp\left[ -
   \eta \hat{L}_t(g_t)
 \right] \right\}.
\end{equation}
Next, we note that
\begin{align*}
 \exp\left\{ - \eta \sum_{t=1}^T \mathbb{E}_{{\hat g}_t \sim \pi_t}[\hat{L}_t(g_t)]
 \right\}
  & = \prod_{t=1}^T  \exp\left\{ - \eta \mathbb{E}_{g_t \sim 
\pi_t}[\hat{L}_t(g_t)]
 \right\} \\
  & \geq  \exp\left(- \frac{T C^2 \eta^2}{8} \right) \prod_{t=1}^T 
 \mathbb{E}_{{\hat g}_t \sim \pi_t} \left\{ \exp\left[ -
   \eta \hat{L}_t(g_t)
 \right] \right\} \\
 & \quad  \quad \text{ (using~\eqref{hoeffding})}\\
 & =    \exp\left\{ -\frac{T C^2 \eta^2}{8}
 \right\}  \prod_{t=1}^T \int 
 \exp\left\{ -
   \eta \hat{L}_t(g)
   \right\} \pi_t({\rm d} g)
  \\
 & = \exp\left\{- \frac{T C^2 \eta^2}{8}
 \right\}  \prod_{t=1}^T \int 
 \frac{\exp\left\{ -
   \eta \sum_{u=1}^t \hat{L}_u(g)
   \right\}}{W_{t}} \pi_1({\rm d} g) \\
  & \quad \quad \text{ (using~\eqref{renorm})}
 \\
 & = \exp\left\{ -\frac{T C^2 \eta^2}{8}
 \right\}  \prod_{T=1}^T \frac{W_{t+1}}{W_{t}}
 = \exp\left\{ \frac{T C^2 \eta^2}{8}
 \right\} W_{T+1}.
\end{align*}
So
\begin{align*}
 \sum_{t=1}^T \mathbb{E}_{{\hat g}_t \sim \pi_t}[\hat{L}_t(g_t)]
 & \leq  - \frac{\log W_{T+1}}{ \eta}
     + \frac{T C^2 \eta}{8}
 \\
 & = - \frac{\log \int \exp\left[-\eta \sum_{t=1}^{T} \hat{L}_t
         (g) \right] \pi_1({\rm d}g)}{ \eta}
     + \frac{T C^2 \eta}{8}
\end{align*}
and finally we use \cite[Equation (5.2.1)]{catoni2004statistical}
which states that
$$
 - \frac{\log \int \exp\left[-\eta \sum_{t=1}^{T} \hat{L}_t
         (g) \right] \pi_1({\rm d}g)}{ \eta}
 = \inf_{\rho} \left\{ \mathbb{E}_{g\sim \rho}\left[ \sum_{t=1}^{T} \hat{L}_t
         (g) \right] + \frac{\mathcal{K}(\rho,\pi_1)}{\eta} \right\}.
$$
\end{proof}

\begin{proof}[Proof of Theorem~\ref{coro-dico}]
Let $D^*$ denote a minimizer to the optimization problem
$$
\min_{D\in\mathcal{D}_K}\frac{1}{T}
\sum_{t=1}^T\inf_{h_t \in\mathcal{H}}
\frac{1}{m}\sum_{i=1}^{m}
 \ell (\langle h_t, D x_{t,i}\rangle,y_{t,i}).
$$
We apply Theorem~\ref{thm:online:w:online} and 
upper bound the infimum with respect to any $\rho$ by 
an infimum with respect to $\rho$ in the following 
parametric family 
\begin{align*}
\rho_{c}({\rm d} D)
\propto \mathbf{1}\{ \forall j=1,\ldots,K:\|D_{\cdot,j}-D_{\cdot,j}^*\| \leq c\}\pi_1(dD).
\end{align*}
where $c$ is a positive parameter. Note that when $c$ is small, $\rho_c$ highly concentrates around $D^*$, but we will show this is at a price of an increase in ${\cal K}(\rho_c,\pi_1)$. The proof then proceeds in optimizing with respect to $c$.

We have that
\begin{multline*}
 \frac{1}{T} \sum_{t=1}^T \mathbb{E}_{{\hat g}_t \sim \pi_t}\left[
 \frac{1}{m}
 \sum_{i=1}^{m} \hat{\ell}_{t,i}
 \right]
 \\
 \leq
 \inf_{c} \Bigg\{
       \mathbb{E}_{D \sim \rho_{c}}\bigg[
        \frac{1}{T}       \sum_{t=1}^T
       \inf_{h_t\in\mathcal{H}}
       \frac{1}{m}       \sum_{i=1}^{m}
        \ell(\langle h_t, D x_{t,i}\rangle,y_{t,i})
+  \beta(m)
       \bigg]
     + \frac{\eta C^2}{8} + \frac{\mathcal{K}(\rho_{c},\pi_1)}{\eta T}
 \Bigg\}.
 \end{multline*}
Now, we have 
$$\mathcal{K}(\rho_{c},\pi_1) 
= - \log \pi_1({\{ \forall j=1,\ldots,K:\|D_{\cdot,j}-D_{\cdot,j}^*\| \leq c\}}),$$
and
\begin{align*}
 \pi_1(\{ \forall j=1,\ldots,K: \|D_{\cdot,j}-D_{\cdot,j}^*\| \leq c\}) 
 & \geq   
   \prod_{j=1}^K \left( \dfrac{\pi^{(d-1)/2}
 (c/2)^{d-1} }{ \Gamma(\frac{d-1}{2}+1)}     \Bigg/  
   \dfrac{2 \pi^{(d+1)/2}}{\Gamma(\frac{d+1}{2})} \right)  \\
   & \geq 
  \prod_{j=1}^K \left(  \dfrac{c^{d-1}}{2^{d}\pi} \right) 
\end{align*}
where the first inequality follows by observing that, since $\pi_1 $ is the 
uniform distribution on the unit $d$-sphere, the probability to 
be calculated is greater or equal to the ration between the volume of the 
$(d-1)$-ball with radius $c/2 $ and the surface area of the
unit $d$-sphere. So we get
 \begin{align*}
\mathcal{K}(\rho_{c},\pi_1) 
 \leq  Kd \log(1/c) + 3Kd.
\end{align*}
Furthermore, using the notation
\begin{align*}
h_t^* &:= {\rm arg} \inf_{h_t \in\mathcal{H}}
\frac{1}{m}\sum_{i=1}^{m}
 \ell \big(\langle h_t, D^* x_{t,i}\rangle,y_{t,i}\big)  ,
\end{align*}
we get
\begin{multline*}
  \inf_{h_t \in\mathcal{H}}
\frac{1}{m}\sum_{i=1}^{m}
 \ell \big(\langle h_t, D x_{t,i}\rangle,y_{t,i}\big)
 -
\frac{1}{m}\sum_{i=1}^{m}
 \ell\big( \langle h^*_t, D^* x_{t,i}\rangle,y_{t,i}\big)
 \\
\leq  
 \frac{1}{m}\sum_{i=1}^{m}
 \ell\big(\langle h^*_t, D x_{t,i}\rangle,y_{t,i}\big)
 -
\frac{1}{m}\sum_{i=1}^{m}
 \ell\big( \langle h^*_t, D^* x_{t,i}\rangle,y_{t,i}\big).
\end{multline*}
Under the condition on the loss, we have
$$
\Big| \ell(\langle h^*_t, D x_{t,i}\rangle,y_{t,i})
 -
 \ell\big( \langle h^*_t, D^* x_{t,i}\rangle,y_{t,i}\big) \Big|
\leq
 \Phi\,   \Big| \left\langle h^*_t, (D-D^*) x_{t,i}\right\rangle\Big|
$$
where $\|\cdot\|_F$ denotes the Frobenius norm.
We obtain an upper-bound
\begin{multline*}
 \mathbb{E}_{D \sim \rho_{c}}
        \frac{1}{T}       \sum_{t=1}^T
       \inf_{h_t\in\mathcal{H}}
       \frac{1}{m}       \sum_{i=1}^{m}
        \ell(\langle h_t, D x_{t,i}\rangle,y_{t,i})
 \\
 \leq
\inf_{D\in\mathcal{D}_K} \Biggl\{\frac{1}{T}
\sum_{t=1}^T\inf_{h_t \in\mathcal{H}}
\frac{1}{m}\sum_{i=1}^{m}
 \ell(\langle h_t, D x_{t,i}\rangle,y_{t,i})
 + \frac{1}{T}       \sum_{t=1}^T
 \frac{1}{m}       \sum_{i=1}^{m}
 \Phi\,   | \left\langle h^*_t, (D-D^*) x_{t,i}\right\rangle|
 \Biggr\}.
\end{multline*}
But then note that
\begin{align*}
 \frac{1}{T}       \sum_{t=1}^T
& \frac{1}{m}       \sum_{i=1}^{m}
 \Phi\,   | \left\langle h^*_t, (D-D^*) x_{t,i}\right\rangle|
 \\
 & =  \frac{1}{T}       \sum_{t=1}^T
 \frac{1}{m}       \sum_{i=1}^{m}
 \Phi\,   \sqrt{ \left\langle h^*_t, (D-D^*) x_{t,i}\right\rangle^2 }
 \\
 & \leq
 \Phi \sqrt{ \frac{1}{T}       \sum_{t=1}^T
 \frac{1}{m}       \sum_{i=1}^{m}
  \left\langle h^*_t, (D-D^*) x_{t,i}\right\rangle^2 } \text{ (Jensen)}
 \\
 & =
  \Phi \sqrt{ \frac{1}{T}       \sum_{t=1}^T
   (h^*_t)^T (D-D^*) \left( \frac{1}{m}       \sum_{i=1}^{m}
   x_{t,i} x_{t,i}^T \right) (D-D^*)^T h^*_t
  }
  \\
 & \leq
  \Phi \sqrt{ \frac{1}{T}       \sum_{t=1}^T
   \lambda_{\max} \left( \frac{1}{m}       \sum_{i=1}^{m}
   x_{t,i} x_{t,i}^T \right) \| (D-D^*)^T h^*_t \|^2
  }
    \\
 & \leq
  \Phi c B \sqrt{ \frac{1}{T}       \sum_{t=1}^T
   \lambda_{\max} \left( \frac{1}{m}       \sum_{i=1}^{m}
   x_{t,i} x_{t,i}^T \right)
  }.
\end{align*}
So
Theorem~\ref{oracle-type} leads to
\begin{multline*}
 \frac{1}{T} \sum_{t=1}^T \mathbb{E}_{g_t \sim
  \pi_t}\left[ \frac{1}{m} \sum_{i=1}^{m} 
  \hat{\ell}_{t,i} \right]
  -
  \inf_{D\in\mathcal{D}_K}\frac{1}{T}
\sum_{t=1}^T\inf_{h_t \in\mathcal{H}}
\frac{1}{m}\sum_{i=1}^{m}
 \ell(\langle h_t, D x_{t,i}\rangle,y_{t,i})
 \\
 \leq
 \inf_{c} \left\{
c \Phi  B \sqrt{ \frac{1}{T}       \sum_{t=1}^T
   \lambda_{\max} \left( \frac{1}{m}       \sum_{i=1}^{m}
   x_{t,i} x_{t,i}^T \right)
  } +  \frac{ Kd }{\eta T} \log(1/c) 
  \right\} +  \frac{ 3Kd }{\eta T}
+ \beta(m) + \frac{\eta C^2}{8} .
\end{multline*}
The choices $c=\sqrt{\frac{1}{T}}$ and 
$\eta = \frac{2}{C}\sqrt{\frac{Kd}{T}} $ lead to the result.
\end{proof}

\begin{proof}[Proof of Theorem~\ref{thm:online:to:batch}]
The proof relies on an application of the well-known online-to-batch
trick, discussed pedagogically in Section 5 page 186
in~\cite{shalev2011online}. Still, it is very cumbersome, and it is
easy to get confused. For these reasons, we think it is important
to write it completely. We use the following notation for any
random variable $V$, $\mathbb{E}_V $ is the expectation with respect to
$V$. This is very important as the online-to-batch trick relies essentially
on inverting the order of the random variables in the integration.
We have:

\begin{align*}
& \mathbb{E}  [\ell(\hat{h}\circ \hat{g}(x),y)]
 \\
 &
 =
 \mathbb{E}_{\mathcal{T}}
\mathbb{E}_{\mathcal{I}}
\mathbb{E}_{P_1,\dots,P_T}
\mathbb{E}_{(x_{j,i},y_{j,i})_{j\leq T,i\leq m}}
\mathbb{E}_{P} \mathbb{E}_{(x_s,y_s)_{s\leq m}}
\mathbb{E}_{(x,y)}
[\ell(\hat{h}\circ \hat{g}(x),y)]
\\
&
=  \frac{1}{T}\sum_{t=1}^T
  \frac{1}{m}\sum_{i=1}^m
  \mathbb{E}_{P_1,\dots,P_T}
\mathbb{E}_{(x_{j,i},y_{j,i})_{j\leq T,i\leq m}}
\mathbb{E}_{P} \mathbb{E}_{(x_s,y_s)_{s\leq m}}
\mathbb{E}_{(x,y)}
[\ell(\hat{h}_i^{\hat{g}_t} \circ \hat{g}_t(x),y)]
\\
&
=  \frac{1}{T}\sum_{t=1}^T
  \mathbb{E}_{P_1,\dots,P_T}
\mathbb{E}_{(x_{j,i},y_{j,i})_{j\leq T,i\leq m}}
\mathbb{E}_{P}
\frac{1}{m}\sum_{i=1}^m
 \mathbb{E}_{(x_s,y_s)_{s\leq i-1}}
\mathbb{E}_{(x,y)}
[\ell(\hat{h}_i^{\hat{g}_t} \circ \hat{g}_t(x),y)]
\\
&
=  \frac{1}{T}\sum_{t=1}^T
  \mathbb{E}_{P_1,\dots,P_T}
\mathbb{E}_{(x_{j,i},y_{j,i})_{j\leq T,i\leq m}}
\mathbb{E}_{P}
\frac{1}{m}\sum_{i=1}^m
 \mathbb{E}_{(x_s,y_s)_{s\leq i-1}}
\mathbb{E}_{(x_i,y_i)}
[\ell(\hat{h}_i^{\hat{g}_t} \circ \hat{g}_t(x_i),y_i)]
\\
&
=  \frac{1}{T}\sum_{t=1}^T
  \mathbb{E}_{P_1,\dots,P_T}
\mathbb{E}_{(x_{j,i},y_{j,i})_{j\leq T,i\leq m}}
\mathbb{E}_{P}
\frac{1}{m}\sum_{i=1}^m
 \mathbb{E}_{(x_s,y_s)_{s\leq m}}
[\ell(\hat{h}_i^{\hat{g}_t} \circ \hat{g}_t(x_i),y_i)]
\\
&
=  \frac{1}{T}\sum_{t=1}^T
  \mathbb{E}_{P_1,\dots,P_T}
\mathbb{E}_{(x_{j,i},y_{j,i})_{j\leq T,i\leq m}}
\mathbb{E}_{P}
 \mathbb{E}_{(x_s,y_s)_{s\leq m}}
\left[
\frac{1}{m}\sum_{i=1}^m
\ell(\hat{h}_i^{\hat{g}_t} \circ \hat{g}_t(x_i),y_i)\right]
\\
&
=  \frac{1}{T}\sum_{t=1}^T
  \mathbb{E}_{P_1,\dots,P_{t-1}}
\mathbb{E}_{(x_{j,i},y_{j,i})_{j\leq t-1,i\leq m}}
\mathbb{E}_{P}
 \mathbb{E}_{(x_s,y_s)_{s\leq m}}
\left[
\frac{1}{m}\sum_{i=1}^m
\ell(\hat{h}_i^{\hat{g}_t} \circ \hat{g}_t(x_i),y_i)\right]
\\
&
=  \frac{1}{T}\sum_{t=1}^T
  \mathbb{E}_{P_1,\dots,P_{t-1}}
\mathbb{E}_{(x_{j,i},y_{j,i})_{j\leq t-1,i\leq m}}
\mathbb{E}_{P_t}
 \mathbb{E}_{(x_s,y_s)_{s\leq m}}
\left[
\frac{1}{m}\sum_{i=1}^m
\ell(\hat{h}_i^{\hat{g}_t} \circ \hat{g}_t(x_{t,i}),y_{t,i})\right]
\\
&
=  \frac{1}{T}\sum_{t=1}^T
  \mathbb{E}_{P_1,\dots,P_T}
  \mathbb{E}_{(x_{j,i},y_{j,i})_{j\leq t,i\leq m}}
  \left[
\frac{1}{m}\sum_{i=1}^m
\ell(\hat{h}_i^{\hat{g}_t} \circ \hat{g}_t(x_{t,i}),y_{t,i})\right]
\\
&
=
  \mathbb{E}_{P_1,\dots,P_T}
  \mathbb{E}_{(x_{j,i},y_{j,i})_{j\leq t,i\leq m}}
  \left[
\frac{1}{T}\sum_{t=1}^T
\frac{1}{m}\sum_{i=1}^m
\ell(\hat{h}_i^{\hat{g}_t} \circ \hat{g}_t(x_{t,i}),y_{t,i})\right]
\\
&
\leq
  \mathbb{E}_{P_1,\dots,P_T}
  \mathbb{E}_{(x_{j,i},y_{j,i})_{j\leq T,i\leq m}}
  \inf_{\rho}
  \Biggl\{
    \mathbb{E}_{g \sim \rho}\Biggl[
        \frac{1}{T}
       \sum_{t=1}^T
       \inf_{h_t\in\mathcal{H}}
       \frac{1}{m}
       \sum_{i=1}^{m}
        \ell(h_t\circ g(x_{t,i}),y_{t,i})
        \\
        & \quad \quad
      +  \frac{1}{T} \sum_{t=1}^T \beta(g,m)
       \Biggr]
     + \frac{\eta C^2}{8} + \frac{\mathcal{K}(\rho,\pi_1)}{\eta T}
  \Biggr\} \text{, using Theorem~\ref{thm:online:w:online},}
\\
&
\leq
  \inf_{\rho}
  \Biggl\{
    \mathbb{E}_{g \sim \rho}\Biggl[
     \mathbb{E}_{P\sim Q}
       \inf_{h_t\in\mathcal{H}}
     \mathbb{E}_{(x,y)\sim P}
        \ell(h_t\circ g(x),y)
      +   \beta(g,m)
       \Biggr]
     + \frac{\eta C^2}{8} + \frac{\mathcal{K}(\rho,\pi_1)}{\eta T}
  \Biggr\}.
\end{align*}
\end{proof}

\section{Better Bounds for Dictionary Learning}
\label{appendix:improvement}

We now state a refined version of the bounds for dictionary learning in Section~\ref{applications}. 
As pointed out in that section, while in general the bound
$$ \lambda_{\max} \left( \frac{1}{m}  \sum_{i=1}^{m} x_{t,i} x_{t,i}^T\right) \leq 1 $$
is unimprovable, if the input vectors $x_{t,i}$ are i.i.d. random variables from uniform distribution on the unit sphere, then
$$ \frac{1}{m}  \sum_{i=1}^{m} x_{t,i} x_{t,i}^T \xrightarrow[m\rightarrow \infty]{a.s.} {\rm Cov}(x_{t,i},x_{t,i}) = \frac{1}{d}I $$
where $I$ is the identity matrix. Consequently,
$$ \lambda_{\max} \left( \frac{1}{m}  \sum_{i=1}^{m} x_{t,i} x_{t,i}^T\right)
 \xrightarrow[m\rightarrow \infty]{a.s.} \frac{1}{d} .$$
We can take advantage of this fact in order to improve the term $\beta(m)=\sup_{g\in\mathcal{G}} \beta(g,m)$, but only if we assume that we know in advance that $\lambda_{\max} \left( \sum_{i=1}^{m} x_{t,i} x_{t,i}^T / m\right)$ is not too large.
This is the meaning of the following theorem.
\begin{thm}
  Assume that we know in advance that for all $t\in\{1,\dots,T\}$,
  $$
  \lambda_{\max} \left( \frac{1}{m}  \sum_{i=1}^{m} x_{t,i} x_{t,i}^T\right)
  \leq \Lambda
  $$
 for some $\Lambda>0$.
 Assume the same assumptions as in Theorem~\ref{coro-dico}, still with
 $ \eta =\frac{2}{C} \sqrt{\frac{Kd}{T}} $. Use within tasks Algorithm~\ref{OGA} (online gradient) with a fixed gradient step $\zeta = B / (L \sqrt{2 m K \Lambda })$.
 Then we have
\begin{multline*}
 \frac{1}{T} \sum_{t=1}^T \mathbb{E}_{g_t \sim
  \pi_t}\left[ \frac{1}{m} \sum_{i=1}^{m} 
  \hat{\ell}_{t,i} \right]
  -  \inf_{g\in\mathcal{G}}\frac{1}{T}
\sum_{t=1}^T\inf_{h_t \in\mathcal{H}}
\frac{1}{m}\sum_{i=1}^{m}
 \ell\big(\langle h_t, g x_{t,i}\rangle,y_{t,i}\big)
 \\
 \leq
\frac{C}{4} \sqrt{\frac{Kd}{T}}( \log(T)+7)
+ \frac{2 B L \sqrt{2 K \Lambda }}{\sqrt{m}}
+ \frac{B\Phi \sqrt{\Lambda}}{\sqrt{T}}
  .
\end{multline*}
\end{thm}
In particular, note that when $\Lambda = 1/d$ the bound becomes
$$
\frac{C}{4} \sqrt{\frac{Kd}{T}}( \log(T)+7)
+ \frac{2 B L \sqrt{2 K }}{\sqrt{m d}}
+ \frac{B\Phi}{\sqrt{d T}}\,.
$$

\begin{proof}
 We apply Theorem~\ref{coro-dico}, so we only have to upper bound the term $\beta(g,m)$ for the online gradient algorithm with the prescribed step size. Note that in~\citep[Corollary 2.7][]{shalev2011online} we actually have the following regret bound for Algorithm~\ref{OGA} with fixed step size $\eta>0$:
 $$
 \beta(g,m) = \frac{B^2}{2\eta m} + \frac{\eta}{m} \sum_{i=1}^m
   \| \nabla_{\theta=\theta_t} \ell(\langle\theta,g x_{t,i}\rangle,y_{t,i}) \|^2.
 $$
 By the $L$-Lipschitz assumption on $\ell$, $
   \| \nabla_{\theta=\theta_t} \ell(\langle\theta_t,g x_{t,i}\rangle,y_{t,i}) \|^2
   \leq L^2 \| g x_{t,i} \|^2$.
 So we have
 \begin{align*}
  \sum_{t=1}^m
   \| \nabla_{\theta=\theta_t} \ell(\langle\theta,g x_{t,i}\rangle,y_{t,i}) \|^2
   & \leq L^2 \sum_{i=1}^m \| g x_{t,i} \|^2
   = L^2 \sum_{i=1}^m \sum_{k=1}^K \langle g_{k,\cdot},x_{t,i} \rangle^2
   \\
 &  \leq  L^2 \sum_{i=1}^m \sum_{k=1}^K  g_{k,\cdot}^T x_{t,i} x_{t,i}^T g_{k\cdot}
 \\
 & \leq m  L^2  \sum_{k=1}^K g_{k,\cdot}^T \left( \frac{1}{m}\sum_{i=1}^m
     x_{t,i} x_{t,i}^T \right) g_{k\cdot}
     \\
 & \leq m K L^2 \lambda_{\max} \left( \frac{1}{m}\sum_{i=1}^m
     x_{t,i} x_{t,i}^T \right) \|g_{k\cdot}\|^2
 \leq m K L^2 \Lambda.
 \end{align*}
 Consequently,  $
 \beta(m) = \sup_{g} \beta(g,m) \leq B^2/(2\eta m) + \eta K L^2 \Lambda$ and
 The choice $\eta \leq B / (L \sqrt{2 m K \Lambda }) $ leads to
 $$ \beta(m) = 2 B L \sqrt{2  K \Lambda /m}. $$
\end{proof}

\section{Batch-Within-Online Lifelong Learning}
\label{appendix-bwo}

In this last section of the appendix, we present an alternative approach
for the batch-within-online setting discussed in Section~\ref{notations}.
In this setting, the tasks are presented sequentially, but, for each task $t \in \{1,\dots,T\}$ the dataset ${\cal S}_t$ 
 is presented all at once and we assume it is obtained i.i.d. from a distribution $P_t$. Unlike to the reasoning in Section~\ref{online-to-batch}, where we assumed that the $P_t$ were i.i.d. from a distribution $Q$, here we make no assumptions on the generation process underlying the $P_t$'s, which may even be adversarial chosen.

Let us recap the setting. At each time $t\in\{1,\dots,T\}$, a task is presented to the
learner in the following manner:
\begin{enumerate}
 \item nature choses $P_t$ , no assumption is made on this choice. This $P_t$
 is not revealed to the forecaster.
 \item nature draws the sample $ \mathcal{S}_t =
 \big((x_{t,1},y_{t,1}),\dots,(x_{t,m_t},y_{t,m_t})]$
 i.i.d. from $P_t$, and this sample is revealed to the forecaster.
\item based on her/his current guess $\tilde{g}_t$ of $g$ and on the sample
$\mathcal{S}_t$, the forecaster has to run her/his favourite learning algorithm $\hat{h}$ on $(\tilde{g}_t,\mathcal{S}_t)$ to get an estimate
 $ \tilde{h}_{t} = \hat{h}(\tilde{g}_t,\mathcal{S}_t) $
 based on an algorithm of his choice.
 Note that the forecaster
 observes $\tilde{r}_t := r_{t}(\tilde{h}_t \circ \tilde{g}_t)$
 where
 $$ r_t(f) = \frac{1}{m_t}\sum_{i=1}^{m_t} \ell\big(f(x_{t,i}),y_{t,i}\big).
 $$
 \item the forecaster incur the loss $
 R_t(\tilde{h}_t\circ \tilde{g}_t)$ where
 $$ R_t (f) = \mathbb{E}_{(x,y)\sim P_t} \big[\ell\big(f(x),y\big)\big] .$$
 Unfortunately, this quantity is not known to the forecaster.
\end{enumerate}

At the end of time, we are interested in a strategy such that
the compound regret
$$ \mathcal{R} := \frac{1}{T}\sum_{t=1}^T R_t(\tilde{h}_t\circ \tilde{g}_t)
   - \inf_{g\in\mathcal{G}} \frac{1}{T}\sum_{t=1}^T \inf_{h_t\in\mathcal{H}}
   R_t(h_t \circ g)
   $$
is controled.
The situation is similar to the setting discussed in the core of the paper:
we will propose an EWA algorithm for transfer learning, EWA-TL,
for which the regret will be controlled, on the condition that
the learner chooses a suitable within task algorithm. In the online
case, the within tasks algorithm was either EWA or OGA.
In Subsection~\ref{wta} we discuss briefly the within task algorithm.
In Subsection~\ref{tla} we present the EWA-TL algorithm and its theoretical
analysis.

\subsection{Within-task Algorithms}
\label{wta}
We make an additional assumption, that is that the estimator $\hat{h}$ satisfies
a bound in probability:
\begin{multline}
\label{assumption}
 \mathbb{P}
\Biggl[
\forall g\in\mathcal{G},
| r(\hat{h}(g,\mathcal{S}_t)\circ g) - R_t(\hat{h}(g,\mathcal{S}_t)\circ g) |
 \leq \delta(g,m_t,\varepsilon)
 \\
 \text{and }
 \\
| R_t(\hat{h}(g,\mathcal{S}_t)\circ g) - \inf_{h\in\mathcal{H}} R_t(h\circ g) |
\leq 2 \delta(g,m_t,\varepsilon)
 \Biggr] \geq 1-\varepsilon. 
\end{multline}

 In classification, when $\ell$ is the 0-1 loss function, and for any
 $g$, the family $\{h\circ g,h\in\mathcal{H}\}$ has a Vapnik-Chervonenkis 
dimension
 bounded by $V$, then the empirical risk minimizer (ERM)
 $$ \hat{h}(g,\mathcal{S}_t) = \arg\min_{h\in\mathcal{H}} r_t(h\circ g) $$
 satisfies the above condition with
 $$ \delta(g,m_t,\varepsilon) = 2\sqrt{2 \frac{V\log\left(\frac{2m_t {\rm e} }
 {V}\right) 
  + \log\left(\frac{4}{\varepsilon}\right)}{m_t}} ,$$
  see e.g.~\citep[Chapter 4, page 94][]{VC}.
  %Under a low-noise assumption, faster rates in $1/m_t$ are obtained.
Similar rates can be obtained with PAC-Bayesian 
bounds~\citep{mcallester1998some,catoni2004statistical}, but we postpone the details
to future work.

\subsection{EWA-TL}
\label{tla}

\begin{algorithm}[H]
\caption{EWA-TL}
\begin{description}
\item[Data] A sequence of datasets
\\ $ \mathcal{S}_t = \big((x_{t,1},y_{t,1}),\dots,(x_{t,m_t},y_{t,m_t})\big) $,
$1\leq t\leq T$, associated with different learning tasks; the datasets are revealed sequentially, but the points within each dataset $\mathcal{S}_t$ are revealed all at once.
\item[Input] A prior $\pi_1$, a
learning parameter $\eta>0$ and a learning algorithm $\hat{h}$
which satisfies~\eqref{assumption}.
\item[Loop] For $t=1,\dots,T$
\begin{description}
\item[i] Draw $\hat{g}_t \sim \pi_t$.
\item[ii] Run the within-task learning algorithm $\hat{t}$ on
$ \mathcal{S}_t $ to get $\tilde{h}_t = \hat{h}(\hat{g}_t,\mathcal{S}_t) $.
\item[iii] Update
 $$
 \pi_{t+1}({\rm d}g)
\propto \exp\Biggl\{-\eta
\Bigl[
r_t(\hat{h}(\mathcal{S}_t,g)\circ g)
+ \delta(g,m_t,\varepsilon/T)
\Bigr]
\Biggr\}
\pi_{t-1}({\rm d}g).
 $$
\end{description}
\end{description}
\label{Al:main}
\end{algorithm}

We now provide a bound on the regret of EWA-TL.
\begin{thm}
\label{thm:batch:w:online}
Under~\eqref{assumption}, and assuming that there is a constant $C$ such
that $0\leq 
r_t(\hat{h}(\mathcal{S}_t,g)\circ g)+ \delta(g,m_t,\varepsilon/T) \leq C$,
with probability at least $1-\varepsilon$,
\begin{multline*}
\sum_{t=1}^T \mathbb{E}_{\tilde{g}_t \sim \pi_{t-1}}
\Bigl[
R_t(\tilde{h}_t\circ \tilde{g}_t)]
\Bigr]
 \leq
      \inf_{\rho} \Biggl\{     \mathbb{E}_{g \sim \rho} \left[\frac{1}{T}
      \sum_{t=1}^T
    \inf_{h\in\mathcal{H}} R_t(h\circ g)
    + \frac{4}{T}\sum_{t=1}^T \delta(g,m_t,\varepsilon/T)
    \right]
    \\
   + \frac{\eta  C^2}{8} + \frac{\mathcal{K}(\rho,\pi_1)}{\eta T} \Biggr\}.
\end{multline*}
\end{thm}

\begin{proof}[Sketch of the proof]
First, follow the proof of Theorem~\ref{thm:online:w:online} to get:
\begin{multline*}
\sum_{t=1}^T \mathbb{E}_{\tilde{g}_t \sim \pi_{t-1}}
\Bigl[
r_t(\tilde{h}_t\circ \tilde{g}_t)]
+  \delta(\tilde{g}_t,m_t,\varepsilon/T)
\Bigr]
  \leq
    \inf_{\rho} \Biggl\{ \sum_{t=1}^T
    \mathbb{E}_{g \sim \rho}
    \Bigl[
    r_t(\tilde{h}_t\circ g)
    + \delta(g,m_t,\varepsilon/T)
    \Bigr]
   \\
   + \frac{\eta T C^2}{8} + \frac{\mathcal{K}(\rho,\pi)}{\eta} \Biggr\}.
\end{multline*}
So, with probability at least $1-\varepsilon$,
\begin{align*}
 \sum_{t=1}^T & \mathbb{E}_{\tilde{g}_t \sim \pi_{t-1}}
\Bigl[
R_t(\tilde{h}_t\circ \tilde{g}_t)]
\Bigr]
\\
& \leq
\sum_{t=1}^T \mathbb{E}_{\tilde{g}_t \sim \pi_{t-1}}
\Bigl[
r_t(\tilde{h}_t\circ \tilde{g}_t)]
+  \delta(\tilde{g}_t,m_t,\varepsilon/T)
\Bigr]
  \\
&\leq
    \inf_{\rho} \left\{ \sum_{t=1}^T
    \mathbb{E}_{g \sim \rho}
    \Bigl[
    r_t(\tilde{h}_t\circ g)
    + \delta(g,m_t,\varepsilon/T)
    \Bigr]
   + \frac{\eta T C^2}{8} + \frac{\mathcal{K}(\rho,\pi_1)}{\eta} \right\}
  \\
&\leq
      \inf_{\rho} \left\{ \sum_{t=1}^T
    \mathbb{E}_{g \sim \rho}
    \Bigl[
    R_t(\hat{h}_t(g,\mathcal{S}_t)\circ g)
    + 2 \delta(g,m_t,\varepsilon/T)
    \Bigr]\
   + \frac{\eta T C^2}{8} + \frac{\mathcal{K}(\rho,\pi_1)}{\eta} \right\}
 \\
 & \leq \inf_{\rho} \left\{     \mathbb{E}_{g \sim \rho} \left[ \sum_{t=1}^T
    \inf_{h\in\mathcal{H}} R_t(h\circ g)
    + 4\sum_{t=1}^T \delta(g,m_t,\varepsilon/T)
    \right]
   + \frac{\eta T C^2}{8} + \frac{\mathcal{K}(\rho,\pi_1)}{\eta} \right\}.
\end{align*}
\end{proof}

%%%%%%%%%%%%%%%%%%%

\end{document}